\documentclass{article}
\usepackage{times}
\usepackage{graphicx}
\usepackage{latexsym}
\usepackage{fullpage}
\usepackage{authblk}

\usepackage{tikz}
\usetikzlibrary{calc,shapes,positioning}
\usetikzlibrary{arrows}
\newcommand{\midarrow}{\tikz \draw[-triangle 90] (0,0) -- +(.1,0);}

\usepackage[latin1]{inputenc}

\usepackage{bm}
\newcommand{\ubar}[1]{\mkern2mu\underline{\mkern-2mu #1\mkern-2mu}\mkern2mu}
\usepackage{mystyle}
\newcommand{\ubm}[1]{\ubar{\bm{#1}}}
\newcommand{\ubmr}[2]{\ubar{\bm{#1}}^{#2}}

\newcommand{\bmtr}[3]{\bm{#1}^{#3}_{#2}}
\newcommand{\smtr}[3]{{#1}^{#3}_{#2}}

\usepackage{amsmath,graphicx}
\RequirePackage{algorithm}
\RequirePackage{algorithmic}

\begin{document}

\title{Powering Hidden Markov Model by Neural Network based Generative Models}
\author[1]{Dong Liu}
\author[1,2]{Antoine Honor\'e}
\author[1]{Saikat Chatterjee}
\author[1]{Lars~K. Rasmussen}
\affil[1]{KTH Royal Institute of Technology, Stockholm, Sweden}
\affil[2]{Karolinska Institute, Stockholm, Sweden}
\affil[]{E-mail: \{doli, honore, sach, lkra\}@kth.se}

\date{}
\maketitle

\begin{abstract}
  Hidden Markov model (HMM) has been successfully used for sequential
  data modeling problems.
  In this work, we propose to power the modeling capacity of HMM by
  bringing in neural network based generative models.
  The proposed model is termed as GenHMM.
  In the proposed GenHMM, each HMM hidden state is associated with a
  neural network based generative model that has
  tractability of exact likelihood and provides efficient likelihood
  computation.
  A generative model in GenHMM consists of a mixture of generators
  that are realized by flow models.
  A learning algorithm for GenHMM is proposed in
  expectation-maximization framework.
  The convergence of the learning GenHMM is analyzed.
  We demonstrate the efficiency of GenHMM by classification tasks on practical sequential data.
\end{abstract}

\section{Introduction}
Sequential data modeling is a challenging topic in pattern recognition and machine learning. For many applications, the assumption of independent and identically distributed (i.i.d.) data points is too strong to model data properly. Hidden Markov model (HMM) is a classic way to model sequential data without the i.i.d. assumption. HMM has been widely used in different practical problems, including applications in reinforcement learning \cite{ding2018reinforcementhmm,levine2018reinforcementReview}, natural language modeling \cite{khan2016survey,Hariyanti_2019}, biological sequence analysis such as proteins \cite{ASHWIN20172} and DNA \cite{ren2015dna}, etc.

A HMM is a statistical representation of sequential data generating process.
Each state of a HMM is associated with a probabilistic model.
The probabilistic model is used to represent the relationship between a state of HMM and sequential data input. The typical way is to use a Gaussian mixture model (GMM) per state of HMM \cite{juang1986maximum}, where GMMs are used to connect states of HMM to sequential data input. GMM based HMM (GMM-HMM) has become a standard model for sequential data modeling, and been employed widely for practical applications, especially in speech recognition \cite{gales2008application,chatterjee2011auditory}.

Given the success of GMM-HMM, it is not efficient for modeling data in nonlinear manifold. Research attempts at training HMM with neural networks have been made to boost the modeling capacity of HMM. A successful work of this track has brought deep neural network (DNN) that is defined by restrictive Boltzmann machines (RBMs) \cite{Hinton2012} into HMM based models \cite{hinton2012deepSpeech,li2013hybrid,Miao2013ImprovingLC}. RBM based HMM is trained with a hierarchical scheme consisting of multiple steps of unsupervised learning, formatting of a classification network and then supervised learning. The hierarchical procedure comes from the empirical expertise in this domain. To be more specific, the hierarchical learning scheme of RBM/DNN based HMM consists of: i) RBMs are trained one after the other in unsupervised fashion, and are stacked together as one deep neural network model, ii) then a final softmax layer is added to the stack of RBMs to represent the probability of a HMM state given a data input, iii) a discriminative training is performed for the final tuning of the model at the final stage.

Another track of related work is hybrid method of temporal neural network models and HMM. In \cite{liu2019lstmHmmHyb,buys2018bridging,vik2016rnnHmm}, a long short-term memory (LSTM) model/recurrent neural network (RNN) is combined with HMM as hybrid. A hierarchical training is carried out by: i) training a HMM first, ii) then doing modified training of LSTM using trained HMM. This hierarchical training procedure is motivated by the intuition of using LSTM or RNN to fill in the gap where HMM can not learn. 


The above works help improve modeling capacity of HMM based models by bringing in neural networks. A softmax layer is usually used to represent probability whenever a conditional distribution is needed. These hierarchical schemes are built based on intuition of domain knowledge. Training of these hierarchical models usually requires expertise in specific areas to be able to proceed with the hierarchical procedure of training and application usage. 

In this work, we propose a generative model based HMM, termed as GenHMM. Specifically, a generative model in our GenHMM is generator-mixed, where a generator is realized by a neural network to help the model gain high modeling capacity.
Our proposed model, GenHMM,
\begin{itemize}
\item has high modeling capacity of sequential data, due to the neural network based generators;
\item is easy to train. Training of GenHMM employs expectation maximization (EM) framework. Therefore, training a GenHMM is as easy as training a GMM-HMM model, while configuration of GenHMM is flexible;
\item is able to compute loglikelihood exactly and efficiently.
\end{itemize}
Instead of using softmax for probability representation, our GenHMM has tractability of exact loglikelihood of given sequential data, which is based on the change of variable formula. To make the loglikelihood computation efficient, neural network based generators of GenHMM are realized as flow models. 

Our contributions in the paper are as follows.
\begin{itemize}
\item Proposing a neural network based HMM for sequential data modeling, i.e. GenHMM. GenHMM has the tractability of exact likelihood.
\item Designing practical algorithm for training GenHMM under EM framework. Stochastic gradient  search in batch fashion is embedded in this algorithm.
\item Giving convergence analysis for GenHMM under the proposed learning algorithm.
\item Verifying the proposed model on practical sequential data.
\end{itemize}

\section{Generator-mixed HMM (GenHMM)}

\begin{figure}[!t]
  \centering
  \begin{tikzpicture}
    \tikzstyle{enode} = [thick, draw=black, ellipse, inner sep = 1pt,  align=center]
    \tikzstyle{nnode} = [thick, rectangle, rounded corners = 2pt,minimum size = 0.8cm,draw,inner sep = 2pt]
    \node[enode] (g1) at (-0.5,1.8) {$p(\bm{x}| s=1; \bm{\Phi}_{1})$};
    \node[enode] (g2) at (-0.5,0.5) {$p(\bm{x}| s=2; \bm{\Phi}_{2})$};
    \node[enode] (gs) at (-0.5, -1.8) {$p(\bm{x}| s=|\Ss|; \bm{\Phi}_{|\Ss|})$};
    \node[enode] (x) at (4.5,1.5){$\ubm{x}\sim p(\ubm{x};\bm{H})$};

    \draw[dotted,line width=2pt] (0,-0.3) -- (0,-1.2);
    \filldraw[->] (1.9, 0.5)circle (2pt) --  (x) ;
    \draw[->] (g1) -- (1.8, 1.8);
    \draw[->] (g2) -- (1.8, 0.5);
    \draw[->] (gs) -- (1.8, -1.8);

    \begin{scope}[xshift=0.5cm, thick, every node/.style={sloped,allow upside down}]
      \node[nnode] (m) at (3.5,-2) {Memory};
      \node[nnode] (a) at (3.5,-0.5) {$\bm{A}$};

      \draw (2.1,0.9)-- (2.2, 0.);
      \draw (2.2,0.)-- node {\midarrow} (2.2,-2);
      \draw (2.2,-2)-- (m);
      \draw (m)-- (5, -2);
      \draw (5, -2)-- node {\midarrow} (5 ,-0.5);
      \draw (5, -0.5) -- (a);
      \draw (a)-- node {\midarrow} (2.2, -0.5);
      \node at (4.8, -1) {$s_{t}$};
      \node at (2.56, -0.25) {$s_{t+1}$};
    \end{scope}
  \end{tikzpicture}
  \caption{HMM model illustration.}\label{fig:hmm}
  \vspace{0.3cm}
\end{figure}
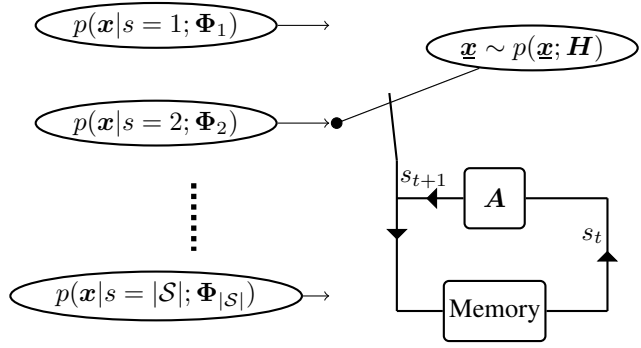

Our framework is a HMM. A HMM $\bm{H}$ defined in a hypothesis space $\Hh$, i.e. $\bm{H} \in \Hh$, is capable to model time-span signal $\ubar{\bm{x}} = \left[ \bm{x}_1, \cdots, \bm{x}_T\right]^{\intercal}$, where $\bm{x}_t\in \RR^{N}$ is the $N$-dimensional signal at time $t$, $[\cdot]^{\intercal}$ denotes transpose, and $T$ denotes the time length\footnote{The length for  sequential data varies.}. We define the hypothesis set of HMM as $\Hh := \{\bm{H} | \bm{H}=\{\Ss, \bm{q}, \bm{A}, p(\bm{x}|{s}; \bm{\Phi}_{s})\}\}$, where
\begin{itemize}
\item $\Ss$ is the set of hidden states of $\bm{H}$.
\item $\bm{q} = \left[ q_1, q_2, \cdots, q_{|\Ss|}\right]^\intercal$ is the initial state distribution of $\bm{H}$ with $|\Ss|$ as cardinality of $\Ss$. For $i \in \Ss$, $q_i = p(s_{1}=i;\bm{H})$. We use $s_t$ to denote the state $s$ at time $t$.
\item $\bm{A}$ matrix of size $|\Ss| \times |\Ss|$ is the transition matrix of states in $\bm{H}$. That is, $\forall i, j \in \Ss$,  $\bm{A}_{i,j} = p(s_{t+1}=j|s_{t}=i; \bm{H})$.
\item For a given hidden state $s$, the density function of the observable signal is $p({\bm{x}}|{s};\bm{\Phi}_{s})$, where $\bm{\Phi}_{s}$ is the parameter set that defines this probabilistic model. Denote $\bm{\Phi} = \left\{ \bm{\Phi}_{s}| s \in \Ss \right\}$.
\end{itemize}

Using HMM for signal representation is illustrated in Figure~\ref{fig:hmm}. The model assumption is that different instant signal of $\ubar{\bm{x}}$ is generated by a different signal source associated with a hidden state of HMM.
In the framework of HMM, at each time instance $t$, signal $\bm{x}_t$ is assumed to be generated by a distribution with density function $p(\bm{x}_t| s_t; \bm{\Phi}_{s_t})$, and $s_t$ is decided by the hidden markov process. Putting these together gives us the probabilistic model $p(\ubm{x};\bm{H})$.

\subsection{Generative Model of GenHMM}

\begin{figure}[!t]
  \centering
  \begin{tikzpicture}
    \tikzstyle{enode} = [thick, draw=black, ellipse, inner sep = 2pt,  align=center]
    \tikzstyle{nnode} = [thick, rectangle, rounded corners = 2pt,minimum size = 0.8cm,draw,inner sep = 2pt]
    \node[enode] (z1) at (-1.2,1.8) {$\bm{z}\sim p_{s,1}(\bm{z})$};
    \node[nnode] (g1) at (1,1.8) {$\bm{g}_{s,1}$};
    \node[enode] (z2) at (-1.2,0.5){$\bm{z}\sim p_{s,2}(\bm{z})$};
    \node[nnode] (g2) at (1,0.5) {$\bm{g}_{s,2}$};
    \node[enode] (zK) at (-1.2,-1.8) {$\bm{z}\sim p_{s,K}(\bm{z})$};
    \node[nnode] (gs) at (1, -1.8) {$\bm{g}_{s,K}$};
    \node[enode] (x) at (4.5,0){$\bm{x}\sim p(\bm{x}| s; \bm{\Phi}_{s})$};

    \draw[dotted,line width=2pt] (0,-0.3) -- (0,-1.2);
    \filldraw[->] (1.9, 0.5)circle (2pt) --  node[above=0.2]{${\kappa}\sim \bm{\pi}_{s}$} (x)  ;
    \draw[->] (z1) -- (g1);
    \draw[->] (g1) -- (1.8, 1.8);

    \draw[->] (z2) -- (g2);
    \draw[->] (g2) -- (1.8, 0.5);

    \draw[->] (zK) -- (gs);
    \draw[->] (gs) -- (1.8, -1.8);
  \end{tikzpicture}
  \caption{Source of state $s$ in GenHMM.}
  \label{fig:gen-mix}
  \vspace{0.3cm}
\end{figure}
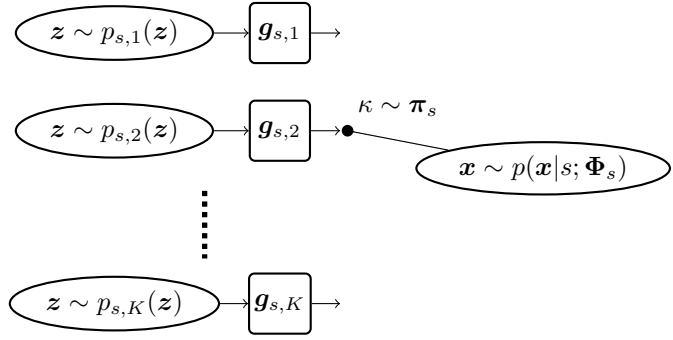

In this section, we introduce the neural network based state probabilistic model of our GenHMM.
Recall that $\bm{x}\in\mathbb{R}^N$. Subscript is omitted when it does not cause ambiguity.
The probabilistic model of GenHMM for each hidden state is a mixture of $K$ neural network based generators, where $K$ is a positive integer.
The probabilistic model of a state $s\in\Ss$ is then given by
\begin{equation}\label{eq:state-prob-model}
  p(\bm{x}| s; \bm{\Phi}_{s}) = \sum_{\kappa=1}^{K}\pi_{s, \kappa} p(\bm{x}| s, \kappa; \bm{\theta}_{s, \kappa}),
\end{equation}
where $\kappa$ is a random variable following a categorical distribution, with probability $\pi_{s, \kappa} = p(\kappa | s; \bm{H})$.
Naturally $\sum_{\kappa = 1}^{K} \pi_{s, \kappa}= 1$. Denote $\bm{\pi}_{s} = [\pi_{s,1}, \pi_{s,2}, \cdots, \pi_{s,K}]^{\intercal}$. 
In \eqref{eq:state-prob-model}, $p(\bm{x}| s, \kappa; \bm{\theta}_{s, \kappa})$ is defined as induced distribution by a generator $\bm{g}_{s,\kappa}: \RR^{N}\rightarrow\RR^{N}$, such that $\bm{x}=\bm{g}_{s, \kappa}(\bm{z})$, where $\bm{z}$ is a latent variable following a distribution with density function $p_{s,\kappa}(\bm{z})$. Generator $\bm{g}_{s,\kappa}$ is parameterized by $\bm{\theta}_{s, \kappa}$. Let us denote the collection of the parameter sets of generators for state $s$ as $\bm{\theta}_s = \left\{ \bm{\theta}_{s, \kappa}| \kappa = 1, 2, \cdots, K \right\}$. Assuming $\bm{g}_{s, \kappa}$ is invertible, by change of variable, we have
\begin{equation}\label{eq:changel-variable}
  p(\bm{x}| s, \kappa; \bm{\theta}_{s, \kappa}) = p_{s,\kappa}(\bm{z})\bigg| \det\left( \pd{\bm{g}_{s,\kappa}(\bm{z})}{\bm{z}} \right)\bigg|^{-1}.
\end{equation}

The signal flow of the probability distribution for a state $s$ of GenHMM is shown in Figure~\ref{fig:gen-mix}, in which the generator identity is up to the random variable $\kappa$.

\subsection{Learning in EM framework}

Assume the sequential signal $\ubm{x}$ follows an unknown distribution $p(\ubm{x})$. We would like to use GenHMM to model this distribution. Alternatively, we are looking for the answer to the question
\begin{equation}
  \umin{\bm{H}\in \Hh} KL({p}(\ubm{x})\| p(\ubm{x};\bm{H})),
\end{equation}
where $KL(\cdot\|\cdot)$ denotes the Kullback-Leibler divergence. For practical consideration, we only have access to the samples of $p(\ubm{x})$, i.e. the dataset of this distribution. For the given dataset, we denote its empirical distribution by $\hat{p}(\ubm{x}) = \frac{1}{R}\sum_{r=1}^{R} \delta_{\ubmr{x}{r}}(\ubm{x})$, where $R$ denotes the total number of sequential samples and superscipt $(\cdot)^{r}$ denotes the index of $r$-th sequential signal. 
The KL divergence minimization problem can be reduced to a likelihood maximization problem
\begin{equation}\label{eq:ml-of-hmm}
  \uargmax{\bm{H} \in \Hh} \frac{1}{R}\sum_{r=1}^{R}\log\,p(\ubmr{x}{r}; \bm{H}).
\end{equation}

For the likelihood maximization, the first problem that we need to address is to deal with the hidden sequential variables of model $\bm{H}$, namely $\ubm{s}=[ \bm{s}_1, \bm{s}_2, \cdots, \bm{s}_T ]^{\intercal}$ and $\ubm{\kappa} = [\bm{\kappa}_1, \bm{\kappa}_2, \cdots, \bm{\kappa}_T]^{\intercal}$. For a sequential observable variable $\ubm{x}$, $\ubm{s}$ is the hidden state sequence corresponding to $\ubm{x}$, and $\ubm{\kappa}$ is the hidden variable sequence representing the generator identity sequence that actually generates $\ubm{x}$.

Since directly maximizing likelihood is not an option for our problem in \eqref{eq:ml-of-hmm}, we address this problem in expectation maximization (EM) framework. This divides our problem into  two iterative steps: i) using the joint posterior of hidden variable sequences $\ubm{s}$ and $\ubm{\kappa}$ to obtain an ``expected likelihood'' of the observable variable sequence $\ubm{x}$, i.e. the E-step; ii) maximizing the expected likelihood with regard to (w.r.t.) the model $\bm{H}$, i.e. the M-step. Assume model $\bm{H}$ is at a configuration of $\bm{H}^{\mathrm{old}}$, we formulate these two steps as follows.
\begin{itemize}
\item E-step: 
  the expected likelihood function
  \begin{equation}\label{eq:em-q-funciton}
    \Qq(\bm{H}; \bm{H}^{\mathrm{old}}) = \EE_{\hat{p}(\ubm{x}),p(\ubm{s},\ubm{\kappa}| \ubm{x}; \bm{H}^{\mathrm{old}})}\left[ \log\,p(\ubm{x}, \ubm{s}, \ubm{\kappa}; \bm{H})\right],
  \end{equation}
  where $\EE_{\hat{p}(\ubm{x}),p(\ubm{s},\ubm{\kappa}| \ubm{x}; \bm{H}^{\mathrm{old}})}\left[ \cdot\right]$ denotes the expectation operator by distribution $\hat{p}(\ubm{x})$ and $p(\ubm{s},\ubm{\kappa}| \ubm{x}; \bm{H}^{\mathrm{old}})$.
\item M-step: the maximization step
  \begin{equation}\label{eq:em-m-opt}
    \umax{\bm{H}} \Qq(\bm{H}; \bm{H}^{\mathrm{old}}).
  \end{equation}
\end{itemize}

The problem \eqref{eq:em-m-opt} can be reformulated as
\begin{align}\label{eq:m-step-subs}
  &\umax{\bm{H}} \Qq(\bm{H}; \bm{H}^{\mathrm{old}}) \nonumber \\
  =&\umax{\bm{q}}\Qq(\bm{q}; \bm{H}^{\mathrm{old}}) + \umax{\bm{A}}\Qq(\bm{A}; \bm{H}^{\mathrm{old}}) 
     + \umax{\bm{\Phi}}\Qq(\bm{\Phi}; \bm{H}^{\mathrm{old}}),
\end{align}
where the decomposed optimization problems are
\begin{align}
  \Qq(\bm{q}; \bm{H}^{\mathrm{old}}) 
    &= \EE_{\hat{p}(\ubm{x}),p(\ubm{s}| \ubm{x}; \bm{H}^{\mathrm{old}})} \left[ \log\,p({s}_{1};\bm{H})  \right], \label{eq:init-distribution-update}\\
  \Qq(\bm{A}; \bm{H}^{\mathrm{old}}) &=\EE_{\hat{p}(\ubm{x}),p(\ubm{s}| \ubm{x}; \bm{H}^{\mathrm{old}})}\hspace{-0.1cm}\left[ \sum_{t=1}^{T-1}\log\,p({s}_{t+1}|{s}_{t}; \bm{H}) \right], \label{eq:transition-update}\\
  \Qq(\bm{\Phi}; \bm{H}^{\mathrm{old}}) &= \EE_{\hat{p}(\ubm{x}),p(\ubm{s},\ubm{\kappa}| \ubm{x}; \bm{H}^{\mathrm{old}})} \left[ \log\,p(\ubm{x}, \ubm{\kappa}| \ubm{s}; \bm{H}) \right]. \label{eq:generative-model-update}
\end{align}

We can see that the solution of $\bm{H}$ depends on the posterior probability $p(\ubm{s}| \ubm{x}; \bm{H})$. Though the evaluation of posterior according to Bayesian theorem is straightforward, the computation complexity of $p(\ubm{s}| \ubm{x}; \bm{H})$ grows exponentially with the length of $\ubm{s}$. Therefore, we employ forward-backward algorithm \cite{Bishop:2006:PRM:1162264} to do the posterior computation efficiently. As we would detail in the next section, what are needed to formulate the problem, are actually the $p(s| \ubm{x}; \bm{H})$ and $p(s, \kappa| \ubm{x}; \bm{H})$. For the joint posterior $p(s, \kappa| \ubm{x}; \bm{H})$, it can be computed by the Bayesian rule when posterior of hidden state is available.

With such a solution framework ready for GenHMM, there are still remaining problems to address before it can be employed for practical usage, including
\begin{itemize}
\item how to realize GenHMM by neural network based generators such that likelihood of their induced distributions can be computed explicitly and exactly?
\item how to train GenHMM to solve problem in \eqref{eq:ml-of-hmm} using practical algorithm?
\item would the training of GenHMM converge?
\end{itemize}
We tackle these problems in the following section.

\section{Solution for GenHMM}

In this section, we detail the solution for realizing and learning GenHMM. 
The convergence of GenHMM is also discussed in this section.

\subsection{Realizing $\bm{g}_{s,\kappa}$ by a Flow Model}
Each generator $\bm{g}_{s,\kappa}$ is realized as a feed-forward neural netowrk.
We define $\bm{g}_{s,\kappa}$ as a $L$-layer neural network and formulate its mapping by layer-wise concatenation:
$\bm{g}_{s,\kappa}=\bm{g}_{s,\kappa}^{[L]}\circ \bm{g}_{s,\kappa}^{[L-1]}\circ \cdots
\circ \bm{g}_{s,\kappa}^{[1]}$, where superscript $[l]$ denotes the layer index and $\circ$ denotes mapping concatenation. Assume $\bm{g}_{s,\kappa}$ is invertible and denote its inverse mapping as $\bm{f}_{s,\kappa}=\bm{g}_{s,\kappa}^{-1}$. For a latent variable $\bm{z}$ with density function $p_{s,\kappa}(\bm{z})$, the generated signal $\bm{x}$ follows an induced distribution with density function \eqref{eq:changel-variable}. We illustrate the signal flow between latent variable $\bm{z}$ and observable variable $\bm{x}$ as 
\begin{equation*}
  \vspace{-8pt}
  \centering
  \begin{tikzpicture}
    \node (z) at (0,0) {};
    \node at ($(z)-(0.5,0)$){$\bm{z}=\bm{h}_0$};
    \node (xi1) at (1.5,0) {$\bm{h}_1$};
    \node (xi2) at (3,0) {};
    \node (xi3) at (4.5,0){};
    \node (x) at (6,0) {};
    \node at ($(x)+(0.5,0)$){${\bm{h}_L=\bm{x}}$};
    \draw[->] ($(z) + (0.3,0.1)$) -- node[above]{$\bm{g}_{s,\kappa}^{[1]}$} ($(xi1)+(-0.3,0.1)$); 
    \draw[->] ($(xi1)-(0.3,0.1)$) -- node[below]{${\bm{f}}_{s,\kappa}^{[1]}$}($(z) - (-0.3,0.1)$);
    \draw[->] ($(xi1) + (0.3,0.1)$) -- node[above]{$\bm{g}_{s,\kappa}^{[2]}$} ($(xi2)+(-0.3,0.1)$); 
    \draw[->] ($(xi2)-(0.3,0.1)$) -- node[below]{${\bm{f}}_{s,\kappa}^{[2]}$}($(xi1) - (-0.3,0.1)$);
    \draw[->] ($(xi3) + (0.3,0.1)$) -- node[above]{$\bm{g}_{s,\kappa}^{[L]}$} ($(x)+(-0.3,0.1)$); 
    \draw[->] ($(x)-(0.3,0.1)$) -- node[below]{${\bm{f}}_{s,\kappa}^{[L]}$}($(xi3) - (-0.3,0.1)$);
    \draw[dotted,line width = 0.3 mm] (xi2) -- (xi3);
  \end{tikzpicture}
\end{equation*}
where $\bm{f}_{s,\kappa}^{[l]}$ is the $l$-th layer of $\bm{f}_{s,\kappa}$. We have $\bm{z}=\bm{f}_{s,\kappa}(\bm{x})$. If every layer of $\bm{g}_{s,\kappa}$ is invertible, the full feed-forward neural network is invertible. Flow model, proposed in \cite{DBLP:journals/corr/DinhKB14} as an image generating model, is such an invertible feed-forward layer-wise neural network. It is further improved in subsquential works \cite{2016arXiv160508803D,kingma2018glow} for high-fidelity and high-resolution image generating and representation. As shown in \eqref{eq:changel-variable}, the challenge lies at the computation of Jacobian determinant. {Another track of flow models uses a continuous-depth models instead. The variable change is defined by an ordinary differential equation implemented by a neural network \cite{NIPS2018_7892, DBLP:grathwohl2018FFJORD}, where the key becomes to solve the ODE problem. We use the layer-wise flow model to model the variable change in \eqref{eq:changel-variable} in which the efficient Jacobian computation is available.}

For a flow model, let us assume that
the feature $\bm{h}_l$ at the $l$'th layer has two subparts as
$\bm{h}_l = [\bm{h}_{l,a}^{\intercal} \, , \, \bm{h}_{l,b}^{\intercal}]^{\intercal}$. The efficient invertible mapping of flow model comes from following forward and inverse relations between $(l-1)$'th and $l$'th layers
\begin{align}\label{eq-gl}
  \bm{h}_{l} &=
               \begin{bmatrix}
                 \bm{h}_{l,a}\\
                 \bm{h}_{l,b}
               \end{bmatrix}
  =
  \begin{bmatrix}
    \bm{h}_{l-1,a}\\
    \left(  \bm{h}_{l-1,b} - \bm{m}_b(\bm{h}_{l-1,a}) \right)\oslash \bm{m}_a(\bm{h}_{l-1,a}) 
  \end{bmatrix}, \nonumber\\
  \bm{h}_{l-1} \hspace{-0.1cm}&=\hspace{-0.1cm}
                                \begin{bmatrix}
                                  \bm{h}_{l-1,a}\\
                                  \bm{h}_{l-1,b}
                                \end{bmatrix}
  \hspace{-0.1cm}=\hspace{-0.1cm}
  \begin{bmatrix}
    \bm{h}_{l,a}\\
    \bm{m}_a(\bm{h}_{l,a})\odot \bm{h}_{l,b} + \bm{m}_b(\bm{h}_{l,a})
  \end{bmatrix},
\end{align}
where $\odot$ denotes element-wise product, $\oslash$ denotes
element-wise division, and $\bm{m}_a(\cdot), \bm{m}_b(\cdot)$ can be
complex non-linear mappings (implemented by neural networks).
For the flow model, the determinant of Jacobian matrix is
\begin{equation}\label{eq:cat-jacobian}
  \begin{array}{rl}
    \mathrm{det}(\nabla{\bm{f}_{s,\kappa}}) = \prod_{l=1}^L \det (\nabla{\bm{f}_{s,\kappa}^{[l]}}),
  \end{array}
\end{equation}
where $\nabla{\bm{f}_{s,\kappa}^{[l]}}$ is the Jacobian of the mapping from the $l$-th layer to the $(l-1)$-th layer, i.e., the inverse transformation. We compute the determinant of the Jacobian matrix as
\begin{align}\label{eq-hl-determinate}
  \det (\nabla{f_{s,\kappa}^{[l]}})& = \det \left[  \pd{\bm{h}_{l-1}}{\bm{h}_l} \right] \nonumber\\
                                   & = \det
                                     \begin{bmatrix}
                                       \bm{I}_a & \mathbf{0} \nonumber\\
                                       \pd{\bm{h}_{l-1,b}}{\bm{h}_{l,a}} & \mathrm{diag}(\bm{m}_a(\bm{h}_{l,a}))
                                     \end{bmatrix}\nonumber\\
                                   &= \det \left( \mathrm{diag}(\bm{m}_a(\bm{h}_{l,a})) \right),
\end{align}
where $\bm{I}_a$ is identity matrix and $\mathrm{diag}(\cdot)$ returns a square matrix with the elements of $\cdot$ on the main diagnal.

\eqref{eq-gl} describes a \textit{coupling} layer in a flow model. A flow model is basically a stack of multiple coupling layers. But the issue of direct concatenation of multiple such coupling mappings is partial identity mapping of the whole model. This issue can be addressed by alternating hidden signal order after each coupling layer.


\subsection{Learning of GenHMM}\label{subsec:optmGenHMM}
In this subsection, we address the problem of learning GenHMM.
\subsubsection{Generative Model Learning}
The generative model learning is actually to solve the problem in \eqref{eq:generative-model-update}, which can be further divided into two subproblems: i) generator learning; ii) mixture weights of generators learning. Let us define notations:
$\bm{\Pi} = \left\{  \bm{\pi}_{s}| s\in \Ss \right\}$, $\bm{\Theta}=\left\{ \bm{\theta}_s| s\in \Ss \right\}$. 
Then the problem in \eqref{eq:generative-model-update} becomes
\begin{align}\label{eq:sub-gm}
  &\umax{\bm{\Phi}} \Qq(\bm{\Phi}; \bm{H}^{\mathrm{old}}) = \umax{\bm{\Pi}} \Qq(\bm{\Pi}; \bm{H}^{\mathrm{old}}) + \umax{\bm{\Theta}} \Qq(\bm{\Theta}; \bm{H}^{\mathrm{old}}),
\end{align}
where
\begin{align}
  \Qq(\bm{\Pi}; \bm{H}^{\mathrm{old}})  &=\EE_{\hat{p}(\ubm{x}),p(\ubm{s},\ubm{\kappa}| \ubm{x}; \bm{H}^{\mathrm{old}})}\left[  \log\,p(\ubm{\kappa}| \ubm{s}; \bm{H})\right], \\
  \Qq(\bm{\Theta}; \bm{H}^{\mathrm{old}}) &=\EE_{\hat{p}(\ubm{x}),p(\ubm{s},\ubm{\kappa}| \ubm{x}; \bm{H}^{\mathrm{old}})}\left[  \log\,p(\ubm{x}| \ubm{s},\ubm{\kappa}; \bm{H})\right].
\end{align}

We firstly address the generator learning problem, i.e. $\umax{\bm{\Theta}} \Qq(\bm{\Theta}; \bm{H}^{\mathrm{old}})$. This is boiled down to maximize the cost function of neural networks that can be formulated as
\begin{align}\label{eq:obj-q-gen-mix-log}
  &\Qq(\bm{\Theta}; \bm{H}^{\mathrm{old}}) \nonumber \\
  = &\frac{1}{R}\sum_{r=1}^{R}\sum_{\ubmr{s}{r}}\sum_{\ubmr{\kappa}{r}}{p(\ubmr{s}{r}, \ubmr{\kappa}{r}| \ubmr{x}{r}; \bm{H}^{\mathrm{old}})} \sum_{t=1}^{{T}^{r}}\log\,p(\bmtr{x}{t}{r} | \smtr{s}{t}{r}, \smtr{\kappa}{t}{r}; \bm{H}) \nonumber \\
  =& \frac{1}{R}\sum_{r=1}^{R} \sum_{t=1}^{{T}^{r}} \sum_{\smtr{s}{t}{r}=1}^{|\Ss|}  \sum_{\smtr{\kappa}{t}{r}=1}^{K}p(\smtr{s}{t}{r}| \ubmr{x}{r}; \bm{H}^{\mathrm{old}})p(\smtr{\kappa}{t}{r}|\smtr{s}{t}{r}, \ubmr{x}{r}; \bm{H}^{\mathrm{old}}) \nonumber\\
  &  \log\, p(\bmtr{x}{t}{r} | \smtr{s}{t}{r}, \smtr{\kappa}{t}{r}; \bm{H}), 
\end{align}
where $T^r$ is the length of the $r$-th sequential data. In \eqref{eq:obj-q-gen-mix-log}, the state posterior $p(s_t| \ubm{x}, \bm{H}^{\mathrm{old}})$ is computed by forward-backward algorithm. The posterior of $\kappa$ is
\begin{align}\label{eq:kappa-posterior}
  p(\kappa| s, \ubm{x}; \bm{H}^{\mathrm{old}})
  &=  \frac{p(\kappa, \ubm{x}| s; \bm{H}^{\mathrm{old}})}{p(\ubm{x}| s,\bm{H}^{\mathrm{old}})} \nonumber \\
  & = \frac{\pi_{s, \kappa}^{\mathrm{old}} p(\bm{x}| s, \kappa, \bm{H}^{\mathrm{old}})}{\sum_{\kappa=1}^{K}  \pi_{s, \kappa}^{\mathrm{old}} p(\bm{x}| s, \kappa,\bm{H}^{\mathrm{old}})},
\end{align}
where the last equation is due to the fact that $\bm{x}_t$ among sequence $\ubm{x}$ only depends on $s_t, \kappa_t$. 

By substituting \eqref{eq:changel-variable} and \eqref{eq:cat-jacobian} into \eqref{eq:obj-q-gen-mix-log}, we have cost function for neural networks as
\begin{align}\label{eq:obj-q-gen-mix}
  &\Qq(\bm{\Theta}; \bm{H}^{\mathrm{old}}) \nonumber \\
  =& \frac{1}{R}\hspace{-3pt}\sum_{r=1}^{R}\hspace{-3pt} \sum_{t=1}^{{T}^{r}}\hspace{-3pt} \sum_{\smtr{s}{t}{r}=1}^{|\Ss|} \hspace{-3pt} \sum_{\smtr{\kappa}{t}{r}=1}^{K}p(\smtr{s}{t}{r}| \ubmr{x}{r}; \bm{H}^{\mathrm{old}})p(\smtr{\kappa}{t}{r}|\smtr{s}{t}{r}, \ubmr{x}{r}; \bm{H}^{\mathrm{old}}) \nonumber\\
  &\left[ \log\, p_{\smtr{s}{t}{r}, \smtr{\kappa}{t}{r}}(\bm{f}_{\smtr{s}{t}{r}, \smtr{\kappa}{t}{r}}(\bmtr{x}{t}{r})) + \sum_{l=1}^{L}\log\,| \det (\nabla{\bm{f}_{s,\kappa}^{[l]}})|\right].
\end{align}
The generators of GenHMM simply use standard Gaussian distribution for latent variables $\bm{z} \sim p_{s,\kappa}(\bm{z})$. Since training dataset can be too large to do whole-dataset iterations, batch-size stochastic gradient decent can be used to maximize $\Qq(\bm{\Theta; \bm{H}^{\mathrm{old}}})$ w.r.t. parameters of generators.

In what follows we address the problem $\max_{\bm{\Pi}} \Qq(\bm{\Pi}; \bm{H}^{\mathrm{old}})$ in our generative model learning. The conditional distribution of hidden variable $\kappa$, $\pi_{s, \kappa} = p(\kappa | s; \bm{H})$, is obtained by solving the following problem
\begin{align}\label{opm:pi}
  \pi_{s, \kappa} & = \uargmax{\pi_{s, \kappa}} \Qq(\bm{\Pi}; \bm{H}^{\mathrm{old}}) \\ \nonumber
                  & s.t. \, \sum_{\kappa=1}^{K} \pi_{s, \kappa}= 1, \forall s = 1, 2, \cdots, |\Ss|. 
\end{align}

To solve problem \eqref{opm:pi}, we formulate its Lagrange function as
\begin{equation}
  \Ll = \Qq(\bm{\Pi}; \bm{H}^{\mathrm{old}}) + \sum_{s=1}^{|\Ss|} \lambda_s\left( 1-  \sum_{\kappa=1}^{K}\pi_{s, \kappa}  \right).
\end{equation}
Solving $\pd{\Ll}{\pi_{s, \kappa}} = 0$ gives
\begin{equation}
  \pi_{s,\kappa} = \frac{1}{\lambda_s}\sum_{r=1}^{R} \sum_{t=1}^{{T}^{r}} p(\smtr{s}{t}{r}=s, \smtr{\kappa}{t}{r} =\kappa| \ubmr{x}{r}; \bm{H}^{\mathrm{old}}).
\end{equation}
With condition $\sum_{\kappa=1}^{K} \pi_{s, \kappa}= 1, \forall s = 1, 2, \cdots, |\Ss|$, we have
\begin{equation}
  \lambda_s = \sum_{\kappa=1}^{K}\sum_{r=1}^{R} \sum_{t=1}^{{T}^{r}} p(\smtr{s}{t}{r}=s, \smtr{\kappa}{t}{r} =\kappa | \ubmr{x}{r}; \bm{H}^{\mathrm{old}}).
\end{equation}
Then the solution to \eqref{opm:pi} is
\begin{equation}\label{eq:mix-latent-parameter-solution}
  \pi_{s, \kappa} = \frac{\sum_{r=1}^{R} \sum_{t=1}^{{T}^{r}} p(\smtr{s}{t}{r} =s, \smtr{\kappa}{t}{r}=\kappa | \ubmr{x}{r}; \bm{H}^{\mathrm{old}}) }{\sum_{k =1}^{K}\sum_{r=1}^{R} \sum_{t=1}^{{T}^{r}} p(\smtr{s}{t}{r} =s, \smtr{\kappa}{t}{r}=k | \ubmr{x}{r}; \bm{H}^{\mathrm{old}}) },
\end{equation}
where
\begin{equation}
  p(s, \kappa | \ubm{x}; \bm{H}^{\mathrm{old}}) = p(s| \ubm{x}; \bm{H}^{\mathrm{old}}) p(\kappa | s, \ubm{x}; \bm{H}^{\mathrm{old}}).
\end{equation}
Here $p(s| \ubm{x}; \bm{H}^{\mathrm{old}})$ can be computed by forward-backward algorithm, while $p(\kappa | s, \ubm{x}; \bm{H}^{\mathrm{old}})$ is given by \eqref{eq:kappa-posterior}.

With the generative model learning obtained, it remains to solve the initial distribution update and transition matrix update of HMM in GenHMM, i.e. the problem \eqref{eq:init-distribution-update} and \eqref{eq:transition-update}. These two problems are basically two constrained optimization problems. The solutions to them are available in literature \cite{Bishop:2006:PRM:1162264}. But to keep learning algorithm for GenHMM complete, we give the update rules for $\bm{q}$ and $\bm{A}$ as follows.

\subsubsection{Initial Probability Update}
The problem in \eqref{eq:init-distribution-update} can be reformulated as
\begin{align}
  &\Qq(\bm{q}; \bm{H}^{\mathrm{old}}) \nonumber \\
  &=\frac{1}{R} \sum_{r=1}^{R}\sum_{\ubmr{s}{r}} {p(\ubmr{s}{r}| \ubmr{x}{r}; \bm{H}^{\mathrm{old}})} \log\,p(\smtr{s}{1}{r};\bm{H}) \nonumber \\
  & = \frac{1}{R}\sum_{r=1}^{R}\sum_{\smtr{s}{1}{r}=1}^{|\Ss|}\sum_{\smtr{s}{2}{r}=1}^{|\Ss|}\cdots \sum_{\smtr{s}{T^{r}}{r}}^{{|\Ss|}} {p(\smtr{s}{1}{r}, \smtr{s}{2}{r}, \cdots, \smtr{s}{T^{r}}{r}| \ubmr{x}{r}; \bm{H}^{\mathrm{old}})} \log\,p(\smtr{s}{1}{r}) \nonumber\\
  & = \frac{1}{R}\sum_{r=1}^{R}\sum_{\smtr{s}{1}{r}=1}^{|\Ss|}{p(\smtr{s}{1}{r}| \ubmr{x}{r}; \bm{H}^{\mathrm{old}})} \log\,p(\smtr{s}{1}{r};\bm{H}).
\end{align}

$p(\smtr{s}{1}{r};\bm{H})$ is the probability of initial state of GenHMM for $r$-th sequential sample. Actually $q_i = p({s}_{1} =i;\bm{H}) $, $i= 1, 2, \cdots, |\Ss|$. Solution to the problem
\begin{align}
  \bm{q} \hspace{-0.1cm} = \hspace{-0.1cm} \uargmax{\bm{q}} \Qq(\bm{q}; \bm{H}^{\mathrm{old}}),\; \mathrm{s.t.} \sum_{i=1}^{ |\Ss| }q_i = 1, q_i \geq 0, \forall i.
\end{align}
is
\begin{equation}\label{eq:update-initial-state-prob}
  q_i = \frac{1}{R} \sum_{r=1}^{R} p(\smtr{s}{1}{r}=i | \ubmr{x}{r}; \bm{H}^{\mathrm{old}}), \forall\; i = 1, 2, \cdots, |\Ss|.
\end{equation}

\subsubsection{Transition Probability Update}
The problem \eqref{eq:transition-update} can be reformulated as
\begin{align}
  &\Qq(\bm{A}; \bm{H}^{\mathrm{old}})\nonumber \\
  &= \sum_{r=1}^{R} \sum_{\ubmr{s}{r}}{p(\ubmr{s}{r}| \ubmr{x}{r}; \bm{H}^{\mathrm{old}})} \sum_{t=1}^{T^{r}-1}\log\,p(\smtr{s}{t+1}{r}|\smtr{s}{t}{r}; \bm{H}) \nonumber \\
  &= \sum_{r=1}^{R} \hspace{-0.1cm}\sum_{t=1}^{T^{r}-1}\hspace{-0.1cm} \sum_{\smtr{s}{t}{r}=1}^{|\Ss|}\hspace{-0.05cm}\sum_{\smtr{s}{t+1}{r}=1}^{|\Ss|}\hspace{-0.2cm}{p(\smtr{s}{t}{r}, \smtr{s}{t+1}{r}| \ubmr{x}{r};\hspace{-0.05cm} \bm{H}^{\mathrm{old}})} \log\,p(\smtr{s}{t+1}{r}|\smtr{s}{t}{r}; \bm{H}).
\end{align}

Since $\bm{A}_{i, j}  = p(\smtr{s}{t+1}{r}=j|\smtr{s}{t}{r}=i; \bm{H})$ is the element of transition matrix $\bm{A}$, the solution to the problem
\begin{align}\label{eq:update-transition-prob}
  \bm{A} = &\uargmax{\bm{A}} \Qq(\bm{A}; \bm{H}^{\mathrm{old}}) \nonumber \\
  \mathrm{s.t.} &\hspace{0.2cm} \bm{A} \cdot \bm{1} = \bm{1}, \bm{A}_{i,j} \geq 0 \,\, \forall i,j,
\end{align}
is
\begin{equation}
  \bm{A}_{i,j} = \frac{\bar{\xi}_{i,j}}{\sum_{k = 1}^{|\Ss|} \bar{\xi}_{i,k}},
\end{equation}
where
\begin{equation}\label{eq:update-transition-solt}
  \bar{\xi}_{i,j} = \sum_{r= 1}^{R} \sum_{t= 1}^{T^{r}-1}{p(\smtr{s}{t}{r}=i, \smtr{s}{t+1}{r}=j| \ubmr{x}{r}; \bm{H}^{\mathrm{old}})}.
\end{equation}
\subsection{On Convergence of GenHMM}
In pursuit of representing a dataset by GenHMM,  we are interested if the learning solution discussed in subsection~\ref{subsec:optmGenHMM} would converge. The properties on GenHMM's convergence are analyzed as follows.

\begin{prop}\label{proposition1}
  Assume that parameter $\bm{\Theta} = \left\{ \bm{\theta}_{s,\kappa}| s\in \Ss, \kappa=1, 2, \cdots, K \right\}$ is in a compact set,  $\bm{f}_{s,\kappa}$ and  ${\nabla\bm{f}_{s,\kappa}}$ are continuous w.r.t. ${\bm\theta}_{s,\kappa}$ in GenHMM. Then GenHMM converges.
\end{prop}

\begin{proof}
  We begin with the comparison of loglikelihood evaluated under $\bm{H}^{\mathrm{new}}$ and $\bm{H}^{\mathrm{old}}$. The loglikelihood of dataset given by $\hat{p}(\ubm{x})$ can be reformulated as
  \begin{align*}
    &\EE_{\hat{p}(\ubm{x})}\left[ \log\,p(\ubm{x};\bm{H}^{\mathrm{new}}) \right] \nonumber \\
    =& \EE_{\hat{p}(\ubm{x}),p(\ubm{s},\ubm{\kappa}| \ubm{x}; \bm{H}^{\mathrm{old}})}\left[ \log\,\frac{p(\ubm{x}, \ubm{s}, \ubm{\kappa}; \bm{H}^{\mathrm{new}})}{p(\ubm{s}, \ubm{\kappa}|\ubm{x}; \bm{H}^{\mathrm{old}})}\right] + \EE_{\hat{p}(\ubm{x})}\left[ KL(p(\ubm{s}, \ubm{\kappa}|\ubm{x}; \bm{H}^{\mathrm{old}})\|p(\ubm{s}, \ubm{\kappa}|\ubm{x}; \bm{H}^{\mathrm{new}})) \right],
  \end{align*}
  where the first term on the right hand side of the above inequality can be further written as
  \begin{align*}
    &\EE_{\hat{p}(\ubm{x}),p(\ubm{s},\ubm{\kappa}| \ubm{x}; \bm{H}^{\mathrm{old}})}\left[ \log\,\frac{p(\ubm{x}, \ubm{s}, \ubm{\kappa}; \bm{H}^{\mathrm{new}})}{p(\ubm{s}, \ubm{\kappa}|\ubm{x}; \bm{H}^{\mathrm{old}})}\right] \nonumber \\
    = &\Qq(\bm{H}^{\mathrm{new}}; \bm{H}^{\mathrm{old}}) + \EE_{\hat{p}(\ubm{x}),p(\ubm{s},\ubm{\kappa}| \ubm{x}; \bm{H}^{\mathrm{old}})}\left[p(\ubm{s}, \ubm{\kappa}|\ubm{x}; \bm{H}^{\mathrm{old}})\right].
  \end{align*}
  According to subsection~\ref{subsec:optmGenHMM}, the optimization problems give
  \begin{align*}
    \Qq(\bm{q}^{\mathrm{new}}; \bm{H}^{\mathrm{old}}) &\geq \Qq(\bm{q}^{\mathrm{old}}; \bm{H}^{\mathrm{old}}),\nonumber \\
    \Qq(\bm{A}^{\mathrm{new}}; \bm{H}^{\mathrm{old}}) &\geq \Qq(\bm{A}^{\mathrm{old}}; \bm{H}^{\mathrm{old}}),\nonumber \\
    \Qq(\bm{\Pi}^{\mathrm{new}}; \bm{H}^{\mathrm{old}}) &\geq \Qq(\bm{\Pi}^{\mathrm{old}}; \bm{H}^{\mathrm{old}}), \nonumber \\
    \Qq(\bm{\Theta}^{\mathrm{new}}; \bm{H}^{\mathrm{old}}) &\geq \Qq(\bm{\Theta}^{\mathrm{old}}; \bm{H}^{\mathrm{old}}).
  \end{align*}
  Since
  \begin{align*}
    \Qq(\bm{H}^{\mathrm{new}}; \bm{H}^{\mathrm{old}}) = &\Qq(\bm{q}^{\mathrm{new}}; \bm{H}^{\mathrm{old}}) + \Qq(\bm{A}^{\mathrm{new}}; \bm{H}^{\mathrm{old}}) + \Qq(\bm{\Pi}^{\mathrm{new}}; \bm{H}^{\mathrm{old}})
                                                          + \Qq(\bm{\Theta}^{\mathrm{new}}; \bm{H}^{\mathrm{old}}),
  \end{align*}
  it gives
  \begin{equation*}
    \Qq(\bm{H}^{\mathrm{new}}; \bm{H}^{\mathrm{old}}) \geq \Qq(\bm{H}^{\mathrm{old}}; \bm{H}^{\mathrm{old}}).
  \end{equation*}
  With the above inequality, and the fact that $\EE_{\hat{p}(\ubm{x}),p(\ubm{s},\ubm{\kappa}| \ubm{x}; \bm{H}^{\mathrm{old}})}\left[p(\ubm{s}, \ubm{\kappa}|\ubm{x}; \bm{H}^{\mathrm{old}})\right]$ is independent of $\bm{H}^{\mathrm{new}}$, we have the inequality 
  \begin{align*}
    &\EE_{\hat{p}(\ubm{x}),p(\ubm{s},\ubm{\kappa}| \ubm{x}; \bm{H}^{\mathrm{old}})}\left[ \log\,\frac{p(\ubm{x}, \ubm{s}, \ubm{\kappa}; \bm{H}^{\mathrm{new}})}{p(\ubm{s}, \ubm{\kappa}|\ubm{x}; \bm{H}^{\mathrm{old}})}\right] 
    \geq \EE_{\hat{p}(\ubm{x}),p(\ubm{s},\ubm{\kappa}| \ubm{x}; \bm{H}^{\mathrm{old}})}\left[ \log\,\frac{p(\ubm{x}, \ubm{s}, \ubm{\kappa}; \bm{H}^{\mathrm{old}})}{p(\ubm{s}, \ubm{\kappa}|\ubm{x}; \bm{H}^{\mathrm{old}})}\right].
  \end{align*}
  Due to $KL(p(\ubm{s}, \ubm{\kappa}|\ubm{x};
  \bm{H}^{\mathrm{old}})\|p(\ubm{s}, \ubm{\kappa}|\ubm{x};
  \bm{H}^{\mathrm{old}}))=0$, we have
  \begin{align*}
    &\EE_{\hat{p}(\ubm{x})}\left[ \log\,p(\ubm{x};\bm{H}^{\mathrm{new}}) \right] \nonumber \\
    \geq & \EE_{\hat{p}(\ubm{x}),p(\ubm{s},\ubm{\kappa}| \ubm{x};
           \bm{H}^{\mathrm{old}})}\left[ \log\,\frac{p(\ubm{x},
           \ubm{s}, \ubm{\kappa}; \bm{H}^{\mathrm{old}})}{p(\ubm{s},
           \ubm{\kappa}|\ubm{x}; \bm{H}^{\mathrm{old}})}\right]
           + \EE_{\hat{p}(\ubm{x})}\left[ KL(p(\ubm{s}, \ubm{\kappa}|\ubm{x}; \bm{H}^{\mathrm{old}})\|p(\ubm{s}, \ubm{\kappa}|\ubm{x}; \bm{H}^{\mathrm{old}})) \right]
      \nonumber \\
    = & \EE_{\hat{p}(\ubm{x})}\left[ \log\,p(\ubm{x};\bm{H}^{\mathrm{old}}) \right].
  \end{align*}
  Since $\bm{f}_{s,\kappa}$ and  ${\nabla\bm{f}_{s,\kappa}}$ are continuous w.r.t. ${\bm\theta}_{s,\kappa}$ in GenHMM, $\EE_{\hat{p}(\ubm{x})}\left[ \log\,p(\ubm{x};\bm{H}) \right]$ is bounded. The above inequality shows $\EE_{\hat{p}(\ubm{x})}\left[ \log\,p(\ubm{x};\bm{H}) \right]$ is non-decreasing in learning of GenHMM. Therefore, GenHMM will converge.
  
\end{proof}

\subsection{Algorithm of GenHMM}
\begin{algorithm}[H]
  \caption{Learning of GenHMM}\label{algo:genhmm}
  \begin{algorithmic}[1]
    \STATE {\bfseries Input:}{
      Empirical distribution $\hat{p}(\bm{x})$ of dataset}\\
    \STATE Initializing $\bm{H}^{\mathrm{old}}, \bm{H} \in \Hh$ gives: \\
    $\bm{H}^{\mathrm{old}} = \{\Ss, \bm{q}^{\mathrm{old}}, A^{\mathrm{old}}, p(\bm{x}|s; \bm{\Phi}_{s}^{\mathrm{old}})\}$, \\
    $\bm{H} = \{\Ss, \bm{q}, A, p(\bm{x}|s; \bm{\Phi}_{s})\}$, \\
    in which generators $\left\{\bm{g}_{s,\kappa}|s\in \Ss, \kappa=1,
      2, \cdots, K \right\}$ are all initialized randomly.
    \STATE $\bm{H}^{\mathrm{old}} \gets \bm{H}$
    \STATE Set learning rate $\eta$, neural network optimization batches $N$ per EM step
    \FOR { $\bm{H}$ not converge}
    \FOR {epoch $n < N$}
    \STATE Sample a batch of data $\left\{ \ubmr{x}{r} \right\}_{r=1}^{R_b}$ from dataset $\hat{p}(\ubm{x})$ with batch size $R_b$

    \STATE Compute posterior $p(\smtr{s}{t}{r}, \smtr{\kappa}{t}{r}| \ubmr{x}{r}; \bm{H}^{\mathrm{old}})$  
    \STATE Formulate loss ${\Qq}\left({\bm{\Theta}}, {\bm{H}}^{\mathrm{old}}\right)$ in \eqref{eq:obj-q-gen-mix}

    \STATE $\partial{\bm{\Theta}} \gets  \nabla_{\bm{\Theta}} {\Qq}\left({\bm{\Theta}},{\bm{H}}^{\mathrm{old}}\right)$
    \STATE $\bm{\Theta} \gets \bm{\Theta} + \eta \cdot \partial{\bm{\Theta}}$
    \ENDFOR
    \STATE $\bm{q} \gets \uargmax{\bm{q}}\, \Qq(\bm{q}; \bm{H}^{\mathrm{old}})$ by \eqref{eq:update-initial-state-prob}
    \STATE $\bm{A} \gets \uargmax{\bm{A}}\Qq(\bm{A}; \bm{H}^{\mathrm{old}})$ by \eqref{eq:update-transition-solt}
    \STATE $\bm{\Pi} \gets \uargmax{\bm{\Pi}}\Qq(\bm{\Phi}; \bm{H}^{\mathrm{old}})$ by \eqref{eq:mix-latent-parameter-solution}
    \STATE $\bm{H}^{\mathrm{old}} \gets \bm{H}$
    \ENDFOR
  \end{algorithmic}
\end{algorithm}

To summarize the learning solution in subsection~\ref{subsec:optmGenHMM}, we wrap our algorithm into pseudocode as shown in Algorithm~\ref{algo:genhmm}. We use Adam \cite{DBLP:journals/corr/KingmaB14} optimizer for optimization w.r.t. the parameters of generators in GenHMM. As shown from line $6$ to $10$ in Algorithm~\ref{algo:genhmm}, the batch-size stochastic gradient decent can be naturally embedded into the learning algorithm of GenHMM.

As described by the pseudocode in Algorithm~\ref{algo:genhmm}, the learning of GenHMM is divided into optimizations w.r.t. to generators' parameters $\bm{\Theta}$, initial probability $\bm{q}$ of hidden state, transition matrix $\bm{A}$, and generator mixture weights $\bm{\Pi}$. Different from the optimization w.r.t. to $\bm{q}$, $\bm{A}$ and $\bm{\Pi}$, which have optimal solutions, generator learning usually cannot give optimal solution to problem $\max_{\bm{\Theta}} \Qq(\bm{\Theta}; \bm{H}^{\mathrm{old}})$. In fact, given that no optimal $\bm{\Theta}$ is obtained, learning of GenHMM can still converge as long as quantity $\Qq(\bm{\Theta}; \bm{H})$ are improving in iterations in Algorithm~\ref{algo:genhmm}, where the inequalities in Proposition~\ref{proposition1} still hold. Therefore optimal $\bm{\Theta}$ in each iteration is not required for convergence of GenHMM as long as the loss in $\eqref{eq:obj-q-gen-mix}$ is getting improved.

\section{Experiments}
To show the validity of our model, we implement our model in PyTorch and test it with sequential data. We first discuss the experimental setups and then show the experimental results. Code for experiments is available at {https://github.com/FirstHandScientist/genhmm}.

\subsection{Experimental Setup}
The dataset used for sequential data modeling and classification is TIMIT where the speech signal is sampled at $16$kHz.
The TIMIT dataset consists of $5300$ phoneme-labeled speech utterances which are partitioned into two sets: {a train set consists of $4620$ utterance, and a test set consists of $1680$ utterances.} There are totally $61$ different types of phones in TIMIT.
We performed experiments in two cases: i) full $61$-phoneme classification case; ii) $39$-phonme classification case, where $61$ phonemes are folded onto $39$ phonemes as described in \cite{Perdigao11}.

For extraction of feature vectors, we use $25$ms frame length and $10$ms frame shift to convert sound track into standard Mel-frequency cepstral coefficients (MFCCs) features. Experiments using the deltas and delta-deltas of the features are also carried out.

Our experiments are performed for: i) standard classification tasks (Table~\ref{tab:acc-classification39f_a}, \ref{tab:acc-classification39f_b}, \ref{tab:acc-classification13f_a}, \ref{tab:acc-classification13f_b}), ii) classification under noise perturbation (table~\ref{tab:acc-classification39f_noise_snr}, \ref{tab:acc-classification39f_noise_type}). The criterion used to report the results includes accuracy, precision and F1 scores.
In all experiments, generators $\left\{\bm{g}_{s,\kappa}|s\in \Ss, \kappa=1, 2, \cdots, K \right\}$ of GenHMM are implemented as flow models. Specifically, our generator structure follows that of a RealNVP described in \cite{2016arXiv160508803D}.
As discussed, the coupling layer shown in \eqref{eq-gl} maps a part of its input signal identically.
The implementation is such that layer $l+1$ would alternate the input signal order of layer $l$ such that no signal remains the same after two consecutive coupling layers.
We term such a pair of consecutive coupling layers as a \textit{flow block}.
In our experiments, each generator $\bm{g}_{s,\kappa}$ consists of four \textit{flow blocks}.
The density of samples in the latent space is defined as Normal,
i.e. $p_{s,\kappa}(\bm{z})$ is the density function of standard
Gaussian. The configuration for each generator is shown as Table~\ref{table:generator-setting}.

\begin{table}
  \caption{Configuration of generators of GenHMM in Experiments}\label{table:generator-setting}
  \centering
  \begin{tabular}{l|c}
    \toprule
    \makecell{Latent distribution $p_{s,\kappa}(\bm{z})$
    \\$s\in \Ss, \kappa=1, 2, \cdots, K $} & Standard Gaussian \\
    \hline
    Number of flow blocks & $4$ \\
    \hline
    Non-linear mapping $\bm{m}_a$, $\bm{m}_b$ & \makecell{Multiple
                                                layer perception, \\
    $3$ layers and with hidden dimension $24$}\\
    \bottomrule
  \end{tabular}
\end{table}

For each GenHMM, the number of states is adapted to the training dataset.
The exact number of states is decided by computing the average length of MFCC frames per phone in training dataset, and clipping the average length into $\left\{ 3,4,5 \right\}$.
Transition matrix $\bm{A}$ is initialized as upper triangular matrix for GenHMM.

\begin{table}
  \caption{Test accuracy table for $39$ dimensional features and folded $39$ phonemes.}\label{tab:acc-classification39f_a}
  \centering  
  \begin{tabular}{l|l|c|c|c}
    \toprule
    {Model} & Criterion &  K=1 &  K=3 &  K=5  \\  \midrule
    \multirow{3}{*}{GMM-HMM}
            & Accuracy    & $62.3\%$ &  $68.0\%$ &  $68.7\%$  \\
            & {Precision} & $67.9\%$ &  $72.6\%$ &  $73.0\%$  \\
            & {F1}        & $63.7\%$ &  $69.1\%$ &  $69.7\%$ \\
    \midrule
    \multirow{3}{*}{GenHMM}
            & Accuracy    & $76.7\%$   & $77.7\%$ &  $77.7\%$ \\ 
            & {Precision} & $76.9\%$   & $78.1\%$ &  $78.0\%$ \\
            & {F1}        & $76.1\%$   & $77.1\%$ &  $77.0\%$\\
    \bottomrule                                                                  
  \end{tabular}
\end{table}

\begin{table}
  \caption{Test accuracy table for $39$ dimensional features and $61$ phonemes.}\label{tab:acc-classification39f_b}
  \centering  
  \begin{tabular}{l|l|c|c|c} \toprule
    {Model} & Criterion & K=1 &  K=3 &  K=5
    \\ \midrule
    \multirow{4}{*}{GMM-HMM}
            & Accuracy    & $53.6\%$ &  $59.6\%$ & $61.9\%$  \\
            & {Precision} & $59.1\%$ &  $63.9\%$ & $65.7\%$ \\
            & {F1}        & $54.7\%$ &  $60.5\%$ & $62.7\%$\\
    \midrule
    \multirow{3}{*}{GenHMM}
            & Accuracy    & $69.5\%$ & $70.6\%$ & $70.7\%$   \\
            & {Precision} & $69.2\%$ & $70.5\%$ & $71.0\%$ \\
            & {F1}        & $68.6\%$ & $69.6\%$ & $69.6\%$\\
    \bottomrule
  \end{tabular}
  \vspace{0.5cm}
\end{table}

\subsection{Experimental Results}

We firstly show the phoneme classification using 39 dimensional MFCC features (MFCC coefficients, deltas, and delta-deltas), to validate one possible usage of our proposed model. Since generative training is carried out in our experiments, GMM-HMM is trained and tested as a reference model in our experiments. Training and testing of GMM-HMM is in the same condition as GenHMMs are trained and tested. Dataset usage for GenMM and GMM-HMM is the same, and number of states for GMM-HMM is the same as that for GenHMM in modeling each phoneme. Apart from setting the reference model, we also run the experiment comparisons with different total number of mixture components.

Table \ref{tab:acc-classification39f_a} and \ref{tab:acc-classification39f_b} shows the results for this experiments, in which we test both the folded $39$-phoneme classification case (the conventional way) in Table~\ref{tab:acc-classification39f_a} and the $61$-phoneme classification case in Table~\ref{tab:acc-classification39f_b}. As shown in both $61$-phoneme and $39$-phoneme cases, GenHMM gets significant higher accuracy than GMM-HMM for the same number of mixture components. The comparisons with regarding to precision and F1 scores show similar trends and also demonstrate significant improvement of GenHMM's performance. As our expectation, GenHMM has better modeling capacity of sequential data since we bring in the neural network based generators into GenHMM, which should be able to represent complex relationship between states of HMM and sequential data. Apart from the gain of using neural network based generative models, there are also increases of accuracy, precision and F1 scores as the number of mixture components in GenHMM is increased from $K=1$ to $K=5$. The sequential dependency of data is modeled by HMM itself, while each state of HMM can have better representation using a mixture probabilistic model if data represented by the state is multi-mode. Comparing the results in $39$-phoneme and $61$-phoneme cases, GenHMM gets higher accuracy for $39$-phoneme classification than it does for $61$-phoneme classification. The total training dataset size remains the same as $61$ phonemes are folded into $39$ phonemes. There are less training data available per phonemes and more classes to be recognized in the $61$-phoneme case, which makes the task more challenging.

\begin{table}
  \caption{Test accuracy table for $13$ dimensional features and folded $39$ phonemes.}\label{tab:acc-classification13f_a}
  \centering  
  \begin{tabular}{l|l|c|c|c}
    \toprule
    {Model} & Criterion & K=1 &  K=3 &  K=5  \\  \midrule
    \multirow{3}{*}{GMM-HMM}
            & Accuracy & $48.5\%$ &  $51.2\%$ &  $52.4\%$  \\
            & Precision& $56.2\%$ &  $58.3\%$ &  $59.5\%$  \\
            & F1       & $50.3\%$ &  $53.0\%$ &  $54.2\%$  \\
    \midrule
    \multirow{3}{*}{GenHMM}
            & Accuracy & $61.1\%$ &  $62.1\%$ & $62.1\%$   \\ 
            & Precision& $61.1\%$ &  $61.9\%$ &  $62.1\%$  \\
            & F1       & $59.7\%$ &  $60.7\%$ &  $60.2\%$  \\

    \bottomrule                                                                  
  \end{tabular}
\end{table}

\begin{table}
  \caption{Test accuracy table for $13$ dimensional features and $61$
    phonemes.}\label{tab:acc-classification13f_b}
  \centering
  \begin{tabular}{l|l|c|c|c}
    \toprule
    {Model} & Criterion &  K=1 &  K=3 &  K=5 \\ \midrule
    \multirow{3}{*}{GMM-HMM}
            & Accuracy & $37.1\%$ &  $40.6\%$ & $42.2\%$  \\
            &Precision & $44.6\%$ &  $47.4\%$ & $48.8\%$  \\
            & F1       & $38.8\%$ &  $42.1\%$ & $43.7\%$ \\
    \midrule
    \multirow{3}{*}{GenHMM}
            & Accuracy & $50.3\%$ & $50.8\%$ & $52.3\%$   \\
            & Precision& $49.3\%$ & $50.9\%$ & $52.1\%$ \\
            &F1        & $47.8\%$ & $48.3\%$ & $49.3\%$ \\
    \bottomrule
  \end{tabular}
  \vspace{0.5cm}
\end{table}

Similar experiments are carried out by using only the MFCC coefficients as feature input (excluding deltas and delta-deltas). The results are shown in Table~\ref{tab:acc-classification13f_a} and \ref{tab:acc-classification13f_b}. The superior performance of GenHMM remains compared with reference model GMM-HMM, with regarding to accuracy, precision and F1 scores. The gain by using mixture generators is also presented in this set of experiments while the difference between $61$-phoneme and $39$-phoneme cases is similar to the set of experiments in Table~\ref{tab:acc-classification39f_a} and \ref{tab:acc-classification39f_b}.

\begin{table}
  \caption{Test accuracy table of perturbation with white noise ($K=3$, folded $39$ phonomes).}
  \label{tab:acc-classification39f_noise_snr}
  \vspace{-0.1cm}
  \centering
  \begin{tabular}{l|l|c|c|c|c}
    \toprule
    \multirow{2}{*}{Model} & \multirow{2}{*}{Criterion} &
                                                          \multicolumn{4}{c}{White Noise SNR} \\
    \cline{3-6}
    
                           && \makecell{15dB} &  \makecell{20dB} &  \makecell{25dB} & \makecell{30dB}  \\
    \midrule
    \multirow{3}{*}{GMM-HMM}
                           & Accuracy & $36.6\%$ &  $44.2\%$ &  $50.8\%$ & $57.1\%$
    \\
                           &Precision & $59.2\%$ &  $64.2\%$ &  $68.4\%$ & $70.6\%$  \\
                           & F1       & $39.9\%$ &  $47.7\%$ &  $53.9\%$ & $59.9\%$ \\
    \midrule
    \multirow{3}{*}{GenHMM}
                           & Accuracy & $52.4\%$ & $62.0\%$ &  $69.7\%$ & $74.3\%$ \\
                           &Precision & $60.0\%$ &  $65.9\%$ &  $71.7\%$ & $74.8\%$  \\
                           & F1       & $52.5\%$ &  $62.0\%$ &  $69.3\%$ & $73.5\%$ \\
    \bottomrule                                                                  
  \end{tabular}
  \vspace{0.1cm}
\end{table}
\begin{table}
  \caption{Test accuracy table of perturbation by different type of noise (SNR=$20$dB, $K=3$, folded $39$ phonomes).}
  \label{tab:acc-classification39f_noise_type}
  \vspace{-0.1cm}
  \centering
  \begin{tabular}{l|l|c|c|c|c}
    \toprule
    \multirow{2}{*}{Model} & \multirow{2}{*}{Criterion} &
                                                          \multicolumn{4}{c}{Noise Type} \\
    \cline{3-6}
    
                           &  &  White &  Pink &  Babble & Volvo  \\
    \midrule
    \multirow{3}{*}{GMM-HMM}
                           & Accuracy & $44.2\%$ &  $48.8\%$ &  $57.7\%$ & $66.6\%$
    \\
                           &Precision & $64.2\%$ &  $66.1\%$ &  $67.0\%$ & $71.9\%$  \\
                           & F1       & $47.7\%$ &  $52.3\%$ &  $59.7\%$ & $67.8\%$ \\
    \midrule
    \multirow{3}{*}{GenHMM}
                           & Accuracy & $62.0\%$ &  $65.1\%$ &  $70.0\%$ & $75.7\%$ \\
                           &Precision & $65.9\%$ &  $67.8\%$ &  $70.4\%$ & $75.9\%$  \\
                           & F1       & $62.0\%$ &  $64.6\%$ &  $69.0\%$ & $75.3\%$ \\
    \bottomrule                                                                  
  \end{tabular}
  \vspace{0.3cm}
\end{table}

Apart from standard classification testing, we also test the robustness of our model to noise perturbations. We train GenHMM with $K=3$ by clean TIMIT training data in the case of folded $39$ phonemes with $39$ dimensional features. The testing dataset is perturbed by either the same type of noise with different signal-to-noise ratio (SNR) as shown in Table~\ref{tab:acc-classification39f_noise_snr}, or different type of noises with the same SNR as shown in Table~\ref{tab:acc-classification39f_noise_type}. The noise data is from NOISEX-92 database. The baseline of these two sets of experiments is the accuracy testing of GenHMM and GMM-HMM on clean testing data in the same experimental condition, where GenHMM has $77.7\%$ and GMM-HMM gets $68.0\%$ as shown in Table~\ref{tab:acc-classification39f_a}. Similar superior performance of GenHMM with regarding to precision and F1 scores is also shown. It is shown in Table~\ref{tab:acc-classification39f_noise_snr} that GMM-HMM's performance degenerates more than GenHMM's performance at the same level of noise perturbation, though the accuracy of both models increases along the increase of SNR. Especially, for SNR=$30$dB, the accuracy of GenHMM drops only about $3\%$ (from $77.7\%$ to $74.3\%$), while GMM-HMM encounters more than $10\%$ decrease (from $68.0\%$ to $57.1\%$) due to the noise perturbation. In Table~\ref{tab:acc-classification39f_noise_type}, the SNR remains constant and GenHMM is tested with perturbation of different noise types. It is shown that GenHMM still remain higher performance scores at different types of noise perturbations than GMM-HMM. Among these four types of noise, white noise shows most significant impact to GenHMM while the impact of volvo noise is negligible.

\section{Conclusion}
In this work, we proposed a generative model based HMM (GenHMM) whose generators are realized by neural networks. We provided the training method for GenHMM. The validity of GenHMM was demonstrated by the experiments of classification tasks on practical sequential dataset. The learning method in this work is based on generative training. For future work, we would consider discriminative training for classification tasks of sequential data.

\section{Acknowledgments}
We would like to thank Dr. Minh Thành Vu for his discussions and comments on the algorithm analysis, which helped improve this paper considerably.
The computations were enabled by resources provided by the Swedish National Infrastructure for Computing (SNIC) at HPC2N partially funded by the Swedish Research Council through grant agreement no. 2016-07213.

\bibliographystyle{plain}
\bibliography{myref}

\end{document}